\documentclass[11pt]{article}

\usepackage{fullpage}
\usepackage{hyperref}
\usepackage{url}

\usepackage{amsfonts}
\usepackage{amsmath}
\usepackage{graphicx}

\usepackage{algorithm}
\usepackage{algorithmic}
\usepackage{ifthen}
\usepackage[algo2e,ruled]{algorithm2e}

\newtheorem{theorem}{Theorem}
\newtheorem{lemma}[theorem]{Lemma}
\newtheorem{fact}[theorem]{Fact}
\newtheorem{corollary}[theorem]{Corollary}

\newtheorem{claim}{Claim}

\newcommand{\qed}{\hfill $\Box$}
\newenvironment{proof}{\par\noindent{\bf Proof.}}{\qed \par\smallskip\noindent}
\newenvironment{proofof}[1]{\par\noindent{\bf Proof of #1.}}{\qed \par\smallskip\noindent}

\newcommand{\field}[1]{\mathbb{#1}}

\newcommand{\E}{\field{E}}
\newcommand{\Ind}[1]{\field{I}{\{#1\}}}

\newcommand{\dt}{\displaystyle}

\newcommand{\uG}{\bar G}

\newcommand{\loss}{\ell}
\newcommand{\hloss}{\widehat{\loss}}

\newcommand{\hp}{\widehat{p}}
\newcommand{\hs}{\widehat{s}}
\newcommand{\hd}{\widehat{d}}
\newcommand{\hP}{\widehat{P}}
\newcommand{\hG}{\widehat{G}}
\newcommand{\hS}{\widehat{S}}

\newcommand{\reach}[1]{\xrightarrow{{#1}}}
\newcommand{\gammab}{\gamma^{(b)}}
\newcommand{\gammabt}{\gamma^{(b_t)}}
\newcommand{\Tb}{T^{(b)}}

\newcommand{\mas}{\mbox{\tt mas}}

\title{\bf From Bandits to Experts: \\ A Tale of Domination and Independence}

\author{
Noga Alon\\
Tel-Aviv University, Israel\\
\texttt{nogaa@tau.ac.il}\\
\and
Nicol\`o Cesa-Bianchi\\
Universit\`a degli Studi di Milano, Italy \\
\texttt{nicolo.cesa-bianchi@unimi.it} \\
\and
Claudio Gentile\\
University of Insubria, Italy\\
\texttt{claudio.gentile@uninsubria.it} \\
\and
Yishay Mansour \\
Tel-Aviv University, Israel\\
\texttt{mansour@tau.ac.il}
}

\begin{document}

\maketitle

\begin{abstract}
We consider the partial observability model for multi-armed bandits, introduced by Mannor and
Shamir~\cite{MS11}. Our main result is a characterization of regret in the
directed observability model in terms of the dominating and
independence numbers of the observability graph. We also show
that in the undirected case, the learner can achieve optimal regret without even accessing
the observability graph before selecting an action. Both results are shown using variants
of the Exp3 algorithm operating on the observability graph in a time-efficient manner.
\end{abstract}

\section{Introduction}
%

Prediction with expert advice ---see, e.g., \cite{LittlestoneWa94,vo90,cb+97,FreundSc95,cbl06}--- is a
general abstract framework for studying sequential prediction
problems, formulated as repeated games between a player and an
adversary. A well studied example of prediction game is the
following: In each round, the adversary privately assigns a loss
value to each action in a fixed set. Then the player chooses an
action (possibly using randomization) and incurs the corresponding
loss. The goal of the player is to control regret, which is defined
as the excess loss incurred by the player as compared to the best
fixed action over a sequence of rounds. Two important variants of
this game have been studied in the past: the expert setting, where
at the end of each round the player observes the loss assigned to
each action for that round, and the bandit setting, where the player
only observes the loss of the chosen action, but not that of other
actions.

Let $K$ be the number of available actions, and $T$ be the number
of prediction rounds. The best possible regret for the expert
setting is of order $\sqrt{T\log K}$. This optimal rate is achieved
by the Hedge algorithm~\cite{FreundSc95} or the Follow the Perturbed
Leader algorithm~\cite{Kalai:05}. In the bandit setting, the optimal
regret is of order $\sqrt{TK}$, achieved by the INF
algorithm~\cite{DBLP:conf/colt/AudibertB09}. A bandit variant of
Hedge, called Exp3~\cite{AuerCeFrSc02}, achieves a regret with a
slightly worse bound of order $\sqrt{TK\log K}$.

Recently, Mannor and Shamir~\cite{MS11} introduced an elegant way
for defining intermediate observability models between the expert
setting (full observability) and the bandit setting (single
observability). An intuitive way of representing an observability
model is through a directed graph over actions: an arc from action
$i$ to action $j$ implies that when playing action $i$ we get
information also about the loss of action $j$. Thus, the expert
setting is obtained by choosing a complete graph over actions
(playing any action reveals all losses), and the bandit setting is
obtained by choosing an empty edge set (playing an action only
reveals the loss of that action).

The main result of~\cite{MS11} concerns undirected observability graphs.
The regret is characterized in terms of the independence number $\alpha$ of the undirected
observability graph.  Specifically, they prove that
$\sqrt{T\alpha\log K}$ is the optimal regret (up to logarithmic
factors) and show that a variant of Exp3, called ELP, achieves this
bound when the graph is known ahead of time, where $\alpha \in \{1,\ldots,K\}$ 
interpolates between full observability ($\alpha=1$ for the clique) and single
observability ($\alpha=K$ for the graph with no edges).
Given the observability graph, ELP runs a linear program to compute
the desired distribution over actions.
In the case when the graph changes over time, and at each time step
ELP observes the current observability graph before prediction,
a bound of $\sqrt{\sum_{t=1}^T\alpha_t\log K}$ is shown, where $\alpha_t$ is
the independence number of the graph at time $t$.
A major problem left open in~\cite{MS11} was the characterization
of regret for directed observability graphs, a setting for which
they only proved partial results.
%

Our main result is a full characterization
(to within logarithmic factors) of regret in the case of directed and
dynamic observability graphs. Our upper bounds are proven using a new algorithm,
called Exp3-DOM. This algorithm is efficient to run even in the dynamic case: it
just needs to compute a small dominating set of the current
observability graph (which must be given as side information) before
prediction.\footnote
{
Computing an approximately minimum dominating set
can be done by running a standard greedy set cover algorithm, see Section \ref{s:prel}.
} 
As in the undirected case, the regret for the directed case is characterized
in terms of the independence numbers of the observability graphs
(computed ignoring edge directions). We arrive at this result by
showing that a key quantity emerging in the analysis of Exp3-DOM can
be bounded in terms of the independence numbers of the graphs. This
bound (Lemma~\ref{l:weightedamlemma} in the appendix) is based on a combinatorial
construction which might be of independent interest.

We also explore the possibility of the learning algorithm receiving
the observability graph only after prediction, and not before.
%
For this setting, we introduce a new variant of Exp3, called Exp3-SET, which achieves
the same regret as ELP for undirected graphs,
but without the need of accessing the current
observability graph before each prediction.
We show that in some random directed graph models Exp3-SET has also
a good performance. In general, we can upper bound the regret of
Exp3-SET as a function of the maximum acyclic subgraph of the
observability graph, but this upper bound may not be tight. 
Yet, Exp3-SET is
much simpler and computationally less demanding than ELP, which
needs to solve a linear program in each round.

There are a variety of real-world settings where partial observability
models corresponding to directed and undirected graphs are applicable. One of them is route selection. 
We are given a graph of possible routes connecting cities: when we select a route $r$ connecting
two cities, we observe the cost (say, driving time or fuel consumption) of the ``edges" 
along that route and, in addition, we have complete
information on any sub-route $r'$ of $r$, but not vice versa. We
abstract this in our model by having an observability graph over routes $r$, and
an arc from $r$ to any of its sub-routes $r'$.

Sequential prediction problems with partial observability 
models also arise in the context of recommendation systems. For
example, an online retailer, which advertises products to users,
knows that users buying certain products are often interested in a
set of related products. This knowledge can be represented as a
graph over the set of products, where two products are joined by an
edge if and only if users who buy any one of the two are likely to
buy the other as well. In certain cases, however, edges have a
preferred orientation. For instance, a person buying a video game
console might also buy a high-def cable to connect it to the TV set.
Vice versa, interest in high-def cables need not
indicate an interest in game consoles.

Such observability models may also arise in the case when a recommendation
system operates in a network of users. For example, consider the 
problem of recommending a sequence of products, or contents, to users 
in a group. Suppose the recommendation system is hosted on an online social 
network, on which users can befriend each other. In this case, it has been
observed that social relationships reveal similarities in tastes and
interests~\cite{said2010social}. However, social links can also be
asymmetric (e.g., followers of celebrities). In such cases,
followers might be more likely to shape their preferences after the
person they follow, than the other way around. Hence, a product liked
by a celebrity is probably also liked by his/her followers, whereas
a preference expressed by a follower is more often specific to that
person.

\section{Learning protocol, notation, and preliminaries}\label{s:prel}
As stated in the introduction, we consider an adversarial multi-armed bandit setting with a finite action set $V = \{1,\dots,K\}$. 
At each time $t=1,2,\dots$, a player (the ``learning algorithm'') picks some action $I_t \in V$ and incurs 
a bounded loss $\loss_{I_t,t} \in [0,1]$. Unlike the standard adversarial bandit 
problem~\cite{AuerCeFrSc02,cbl06}, 
where only the played action $I_t$ reveals its loss $\loss_{I_t,t}$, here we assume 
all the losses in a subset $S_{I_t,t} \subseteq V$ of actions are revealed after $I_t$ is played. 
More formally, the player observes the pairs $(i,\loss_{i,t})$ for each $i \in S_{I_t,t}$. 
We also assume $i\in S_{i,t}$ for any $i$ and $t$, that is, any action reveals its own loss 
when played. Note that the bandit setting ($S_{i,t} = \{i\}$) and the expert 
setting ($S_{i,t} = V$) 
are both special cases of this framework. 
We call $S_{i,t}$ the {\em observation set} of action $i$ at time $t$, and write $i \reach{t} j$
when at time $t$ playing action $i$ also reveals the loss of action $j$. Hence, 
$S_{i,t} = \{j\in V\,:\, i \reach{t} j\}$.
The family of observation sets $\{S_{i,t}\}_{i\in V}$ we collectively call the
{\em observation system} at time $t$.

The adversaries we consider are nonoblivious. Namely, each loss $\loss_{i,t}$ at time $t$ can be an arbitrary 
function of the past player's actions $I_1,\dots,I_{t-1}$. The performance of a player $A$
is measured through the regret
\[
    \max_{k\in V} \E\bigl[L_{A,T} - L_{k,T}\bigl]~,
\]
where $L_{A,T} = \loss_{I_1,1} + \cdots + \loss_{I_T,T}$ and $L_{k,T} = \loss_{k,1} + \cdots + \loss_{k,T}$
are the cumulative losses of the player and of action $k$, respectively. The expectation is taken with 
respect to the player's internal randomization (since losses are allowed to depend on the player's past 
random actions, also $L_{k,t}$ may be random).\footnote
{ 
Although we defined the problem in terms of losses, our 
analysis can be applied to the case when actions return rewards $g_{i,t} \in [0,1]$ via the 
transformation $\loss_{i,t} = 1 - g_{i,t}$.
}
The observation system $\{S_{i,t}\}_{i\in V}$ is either adversarially generated (in which case, 
each $S_{i,t}$ can be an arbitrary function of past player's actions, just like losses are), or
randomly generated ---see Section \ref{s:symm}.
In this respect, we distinguish between {\em adversarial} and {\em random} observation systems.

Moreover, whereas some algorithms need to know the observation system at the beginning 
of each step $t$, others need not. From this viewpoint, we shall consider two online learning
settings. In the first setting, called the {\em informed} setting, the whole observation system 
$\{S_{i,t}\}_{i\in V}$ selected by the adversary is made available to the learner
{\em before} making its choice $I_t$. This is essentially the ``side-information" framework first 
considered in \cite{MS11} 
In the second setting, called the {\em uninformed setting}, no information whatsoever
regarding the time-$t$ observation system is given to the learner prior to prediction.

We find it convenient to adopt the same graph-theoretic interpretation of observation systems
as in \cite{MS11}. At each time step $t=1,2,\dots$, the observation system $\{S_{i,t}\}_{i\in V}$ 
defines a directed graph $G_t = (V,D_t)$, 
where $V$ is the set of actions, and $D_t$ is the set of arcs, i.e., ordered pairs of nodes. 
For $j \neq i$, arc $(i,j) \in D_t$ if and only if $i \reach{t} j$ (the self-loops created by 
$i \reach{t} i$ are intentionally ignored). Hence, we can equivalently define $\{S_{i,t}\}_{i\in V}$
in terms of $G_t$. Observe that the outdegree $d_i^+$ of any $i \in V$ equals $|S_{i,t}|-1$. 
Similarly, the indegree $d_i^-$ of $i$ is the number of action $j \neq i$ such that $i \in S_{j,t}$
(i.e., such that $j \reach{t}i$).
A notable special case of the above is when the observation system is symmetric over time: 
$j \in S_{i,t}$ if and only if $i \in S_{j,t}$ for all $i,j$ and $t$. 
In words, playing $i$ at time $t$ reveals the loss of $j$ if and only if playing $j$ at time $t$ 
reveals the loss of $i$. A symmetric observation system is equivalent to $G_t$ being an undirected graph or,
more precisely, to a directed graph having, for every pair of nodes $i,j \in V$, either no arcs or length-two directed cycles.
Thus, from the point of view of the symmetry of the observation system, we also 
distinguish between the {\em directed} case ($G_t$ is a general directed graph) and the {\em symmetric} case ($G_t$ is an undirected graph for all $t$). For instance, combining the terminology introduced
so far, the adversarial, informed, and directed
setting is when $G_t$ is an adversarially-generated directed graph disclosed to the algorithm in round
$t$ before prediction, while the random, uninformed, and directed setting is when $G_t$ is a randomly generated
directed graph which is not given to the algorithm before prediction.

The analysis of our algorithms depends on certain properties of the sequence of graphs $G_t$.
Two graph-theoretic notions playing an important role here are those of {\em independent sets} 
and {\em dominating sets}. 
Given an undirected graph $G = (V,E)$, an independent set of $G$ is any subset 
$T \subseteq V$ such that no two $i,j \in T$ are connected by an edge in $E$. 
An independent set is {\em maximal} if no proper superset thereof 
is itself an independent set. The size of a largest (maximal) independent set 
is the {\em independence number} of $G$, denoted by $\alpha(G)$.
If $G$ is directed, we can still associate with it an independence number: we simply view $G$
as undirected by ignoring arc orientation.
If $G = (V,D)$ is a directed graph, then a subset $R \subseteq V$ is a 
dominating set for $G$ if for all $j \not\in R$ there exists some $i \in R$ such 
that arc $(i,j) \in D$. In our bandit setting, a time-$t$ dominating set $R_t$ is a subset of 
actions with the property that the loss of any remaining action in round $t$ can be observed by 
playing some action in $R_t$. A dominating set is {\em minimal} if no proper subset thereof 
is itself a dominating set.
The 
domination number of directed graph $G$, denoted by
$\gamma(G)$, is the size of a smallest (minimal) dominating set of $G$. 

Computing a minimum dominating set for an arbitrary directed graph $G_t$ is equivalent to 
solving a minimum set cover problem on the associated observation system $\{S_{i,t}\}_{i\in V}$. 
Although minimum set cover is NP-hard, the well-known Greedy Set Cover algorithm~\cite{Chv79}, 
which repeatedly selects from $\{S_{i,t}\}_{i\in V}$ the set containing the largest number of 
uncovered elements so far, computes a dominating set $R_t$ such that $|R_t| \leq \gamma(G_t)\,(1+\ln K)$.


Finally, we can also lift the independence number of an undirected graph to directed graphs through
the notion of {\em maximum acyclic subgraphs}: Given a directed graph $G = (V,D)$, an acyclic subgraph
of $G$ is any graph $G' = (V',D')$ such that $V'\subseteq V$, and $D'= D\cap \bigl(V'\times V'\bigr)$, with no (directed) cycles.
We denote by $\mas(G)= |V'|$ the maximum size of such $V'$.
Note that when $G$ is undirected (more precisely, as above, when $G$ is a directed graph having for every pair of nodes 
$i,j \in V$ either no arcs or length-two cycles), then $\mas(G) = \alpha(G)$, otherwise $\mas(G) \geq \alpha(G)$. 
In particular, when $G$ is itself a directed acyclic graph, then $\mas(G) = |V|$.

\section{Algorithms without Explicit Exploration: The Uninformed Setting}\label{s:symm}
In this section, we show that a simple variant of the Exp3 algorithm~\cite{AuerCeFrSc02} 
obtains optimal regret (to within logarithmic factors) in two variants of the uninformed setting: 
(1) adversarial and symmetric, (2) random and directed. We then show that even the harder 
adversarial and directed setting lends itself to an analysis, though with a weaker regret bound. 

%
\begin{algorithm2e}[t]
\SetKwSty{textrm} \SetKwFor{For}{{\bf For}}{}{}
\SetKwIF{If}{ElseIf}{Else}{if}{}{else if}{else}{}
\SetKwFor{While}{while}{}{}
\textbf{Parameter:} $\eta \in [0,1]$;\\
\textbf{Initialize:} $w_{i,1} = 1$ for all $i \in V = \{1,\ldots,K\}$;\\
\For{$t=1,2,\dots$:}
{ {
\begin{enumerate}
\item Observation system $\{S_{i,t}\}_{i\in V}$ is generated (but not disclosed)~;
\item Set ${\dt p_{i,t} = \frac{w_{i,t}}{W_{i,t}}}$ for each $i\in V$, where ${\dt W_t = \sum_{j \in V} w_{j,t}}$~;
\item Play action $I_t$ drawn according to distribution $p_t = (p_{1,t},\dots,p_{K,t})$~;
\item Observe pairs $(i,\loss_{i,t})$ for all $i \in S_{I_t,t}$;
\item Observation system $\{S_{i,t}\}_{i\in V}$ is disclosed~;
%
\item For any $i \in V$ set $w_{i,t+1} = w_{i,t}\,\exp\bigl(-\eta\,\hloss_{i,t}\bigr)$, where
\[
    \hloss_{i,t}
=
    \frac{\loss_{i,t}}{q_{i,t}}\,\Ind{i \in S_{I_t,t}}
\qquad\text{and}\qquad
    q_{i,t} = \sum_{j \,:\, j \reach{t} i} p_{j,t}~.
\]
\end{enumerate}
\vspace{-0.2in}
} }
\caption{Exp3-SET: Algorithm for the uninformed setting}
\label{a:lossalg}
\end{algorithm2e}
%
Exp3-SET (Algorithm~\ref{a:lossalg}) runs Exp3 without mixing with the uniform distribution. Similar to Exp3, Exp3-SET uses loss estimates $\hloss_{i,t}$ that divide each observed loss
$\loss_{i,t}$ by the probability $q_{i,t}$ of observing it. This probability $q_{i,t}$ is simply the sum of
all $p_{j,t}$ such that $j \reach{t} i$ (the sum includes $p_{i,t}$). Next, we bound the regret of Exp3-SET
in terms of the key quantity
\begin{equation}\label{e:Qt}
    Q_t = \sum_{i \in V} \frac{p_{i,t}}{q_{i,t}} = \sum_{i \in V} \frac{p_{i,t}}{\sum_{j \,:\, j \reach{t} i} p_{j,t}}~.
\end{equation}
Each term $p_{i,t}/q_{i,t}$ can be viewed as the probability of drawing $i$ from $p_t$ conditioned on the event that $i$ was observed.
Similar to~\cite{MS11}, a key aspect to our analysis is the ability to deterministically (and nonvacuously)\footnote
{
An obvious upper bound on $Q_t$ is $K$.
} 
upper bound $Q_t$ in terms 
of certain quantities defined on $\{S_{i,t}\}_{i\in V}$. 
We shall do so in two ways, either irrespective of how small each $p_{i,t}$ may be (this section) 
or depending on suitable lower bounds on the probabilities $p_{i,t}$ (Section~\ref{s:directed}).
In fact, forcing lower bounds on $p_{i,t}$ is equivalent to adding exploration terms to the algorithm, 
which can be done only when knowing $\{S_{i,t}\}_{i\in V}$ before each prediction ---an information 
available only in the informed setting.


The following simple result is the building block for all subsequent results in the uninformed setting.\footnote
{
All proofs are given in the appendix.
}
\begin{theorem}\label{thm:noexp}
In the adversarial case, the regret of Exp3-SET satisfies
\[
    \max_{k \in V} \E\bigl[L_{A,T} - L_{k,T}\bigr]
\le
    \frac{\ln K}{\eta} + \frac{\eta}{2}\,\sum_{t=1}^T \E[Q_t]~.
\]
\end{theorem}
As we said, in the adversarial and symmetric case the observation system at time $t$ can be described by an
undirected graph $G_t = (V,E_t)$. This is essentially the problem of \cite{MS11}, which they studied in the easier
informed setting, where the same quantity $Q_t$ above arises in the analysis of their ELP algorithm. In their
Lemma~3, they show that $Q_t \le \alpha(G_t)$, irrespective of the choice of the probabilities $p_t$.
When applied to Exp3-SET, this immediately gives the following result.
\begin{corollary}\label{thm:symmetric}
In the adversarial and symmetric case, the regret of Exp3-SET satisfies
\[
    \max_{k \in V} \E\bigl[L_{A,T} - L_{k,T}\bigr]
\le
    \frac{\ln K}{\eta} + \frac{\eta}{2}\,\sum_{t=1}^T \E[\alpha(G_t)]~.
\]
In particular, if for constants $\alpha_1, \ldots, \alpha_T$
we have $\alpha(G_t) \leq \alpha_t$, $t = 1, \ldots, T$, then
setting $\eta = \sqrt{(2\ln K)\big/\sum_{t=1}^T \alpha_t}$, gives
\[
    \max_{k \in V} \E\bigl[L_{A,T} - L_{k,T}\bigr]
\le
    \sqrt{2(\ln K)\sum_{t=1}^T \alpha_t}~.
\]
\end{corollary}
As shown in~\cite{MS11}, the knowledge of $\sum_{t=1}^T\alpha(G_t)$ for tuning $\eta$ can be dispensed
with (at the cost of extra log factors in the bound) by binning the values of $\eta$ and running
Exp3 on top of a pool of instances of Exp-SET, one for each bin.
The bounds proven in Corollary~\ref{thm:symmetric} are equivalent to those proven
in \cite{MS11} (Theorem 2 therein) for the ELP algorithm. Yet, our analysis is much simpler
and, more importantly, our algorithm is simpler and more efficient than ELP, which
requires solving a linear program
at each step. Moreover, unlike ELP, Exp-SET does not require prior knowledge of the observation
system $\{S_{i,t}\}_{i\in V}$ at the beginning of each step.


We now turn to the directed setting. We first treat the random case, and then the harder adversarial case.

The Erd\H{o}s-Renyi model is a standard model for random directed graphs $G = (V,D)$, where
we are given a density parameter $r \in [0,1]$ and, for any pair $i,j \in V$,
arc $(i,j)\in D$ with independent probability
$r$.\footnote
{
Self loops, i.e., arcs $(i,i)$ are included by default here.
} 
We have the following result.
\begin{corollary}\label{thm:random_er}
Let $G_t$ be generated according to the Erd\H{o}s-Renyi model with parameter $r \in [0,1]$.
Then the regret of Exp3-SET satisfies
\[
    \max_{k \in V} \E\bigl[L_{A,T} - L_{k,T}\bigr]
\le
    \frac{\ln K}{\eta} + \frac{\eta\,T}{2r}\left(1 - (1-r)^{K}\right)~.
\]
In the above, the expectations $\E[\cdot]$ are w.r.t.\ both the
algorithm's randomization and the random generation of $G_t$
occurring at each round. In particular, setting $\eta =
\sqrt{\frac{2r\ln K}{T\left(1 - (1-r)^{K}\right)}}$, gives
\[
    \max_{k \in V} \E\bigl[L_{A,T} - L_{k,T}\bigr]
\le
    \sqrt{\frac{2(\ln K)T\left(1 - (1-r)^{K}\right)}{r}}~.
\]
\end{corollary}
Note that as $r$ ranges in $[0,1]$ we interpolate between the bandit ($r=0$)\footnote
{
Observe that $\lim_{r\rightarrow 0^+} \frac{1 - (1-r)^{K}}{r} = K$.
} 
and the expert ($r=1$) regret bounds.

In the adversarial setting, we have the following result.
%
%
%
\begin{corollary}\label{c:ndag}
In the adversarial and directed case, the regret of Exp3-SET satisfies
\[
    \max_{k \in V} \E\bigl[L_{A,T} - L_{k,T}\bigr]
\le
    \frac{\ln K}{\eta} + \frac{\eta}{2}\,\sum_{t=1}^T \E[\mas(G_t)]~.
\]
In particular, if for constants $m_1, \ldots, m_T$ we have
$\mas(G_t) \leq m_t$, $t = 1, \ldots, T$, then setting $\eta =
\sqrt{(2\ln K)\big/\sum_{t=1}^T m_t}$, gives
\[
    \max_{k \in V} \E\bigl[L_{A,T} - L_{k,T}\bigr]
\le
    \sqrt{2(\ln K)\sum_{t=1}^T m_t}~.
\]
\end{corollary}
Observe that Corollary \ref{c:ndag} is a strict generalization of 
Corollary \ref{thm:symmetric} because, as we pointed out in Section~\ref{s:prel},
$\mas(G_t) \geq \alpha(G_t)$, with equality holding when $G_t$ is an undirected graph.

As far as lower bounds are concerned, in the symmetric setting, the authors of \cite{MS11} derive a lower bound of
$\Omega\bigl(\sqrt{\alpha(G)T}\bigr)$ in the case when $G_t=G$ for all $t$.
We remark that similar to the symmetric setting, we can derive a lower
bound of $\Omega\bigl(\sqrt{\alpha(G)T}\bigr)$. The simple observation is that given
a directed graph $G$, we can define a new graph $G'$ which is made undirected 
just by reciprocating arcs; namely, if there is an arc $(i,j)$ in $G$ we add 
arcs $(i,j)$ and $(j,i)$ in $G'$. Note that $\alpha(G)=\alpha(G')$.
Since in $G'$ the learner can only receive more information than in
$G$, any lower bound on $G$ also applies to $G'$. Therefore we
derive the following corollary to the lower bound of~\cite{MS11} (Theorem~4 therein).
\begin{corollary}\label{c:lower}
Fix a directed graph $G$, and suppose $G_t = G$ for all $t$. Then
there exists a (randomized) adversarial strategy such that for any 
$T =\Omega\bigl(\alpha(G)^3\bigr)$ and for any learning strategy, the expected regret
of the learner is $\Omega\bigl(\sqrt{\alpha(G)T}\bigr)$.
\end{corollary}
One may wonder whether a sharper lower bound argument exists which applies to the general
directed setting and involves the larger quantity $\mas(G)$.
Unfortunately, the above measure does not seem to be related to the optimal
regret: Using Claim~\ref{cl:1} in the appendix (see proof of Theorem \ref{thm:random_er}) 
one can exhibit a sequence of graphs each having a large acyclic subgraph, on which the 
regret of Exp3-SET is still small.


The lack of a lower bound matching the upper bound provided by Corollary \ref{c:ndag}
is a good indication that something more sophisticated has to be done in order to upper 
bound $Q_t$ in (\ref{e:Qt}). This leads us to consider more refined ways of allocating 
probabilities $p_{i,t}$ to nodes. However, this allocation will require prior knowledge 
of the graphs $G_t$.

\section{Algorithms with Explicit Exploration: The Informed Setting}\label{s:directed}
We are still in the general scenario where graphs $G_t$ are arbitrary and directed, but now
$G_t$ is made available before prediction. 
%
%
We start by showing a simple example where our analysis of Exp3-SET inherently fails.
This is due to the fact that, when the graph induced by the observation system is directed,
the key quantity $Q_t$ defined in (\ref{e:Qt}) cannot be nonvacuously upper bounded
independent of the choice of probabilities $p_{i,t}$.
A way round it is to introduce a new algorithm, called Exp3-DOM, which controls probabilities $p_{i,t}$
by adding an exploration term to the distribution $p_t$. This exploration term is supported on a
dominating set of the current graph $G_t$. For this reason, Exp3-DOM requires prior access to a
dominating set $R_t$ at each time step $t$ which, in turn, requires prior knowledge of the entire 
observation system $\{S_{i,t}\}_{i\in V}$. 

As announced, the next result shows that, even for simple directed graphs, there exist distributions 
$p_t$ on the vertices such that $Q_t$ is linear in the number of nodes while the independence number 
is $1$.\footnote
{
In this specific example, the maximum acyclic subgraph has size $K$, which confirms the looseness of 
Corollary~\ref{c:ndag}.  
} 
Hence, nontrivial bounds on $Q_t$ can be found only by imposing conditions on distribution $p_t$.
\begin{fact}\label{l:bad}
Let $G = (V,D)$ be a total order on $V = \{1,\dots,K\}$, i.e., such
that for all $i \in V$, arc $(j,i) \in D$ for all $j = i+1,\dots,K$.
Let $p = (p_1, \ldots, p_K)$ be a distribution on $V$ such that
$p_i=2^{-i}$, for $i<K$ and $p_k=2^{-K+1}$. Then
\[
Q = \sum_{i = 1}^K \frac{p_i}{p_i+\sum_{j\,:\,j\reach{}i} p_j} =
\sum_{i = 1}^K \frac{p_i}{\sum_{j=i}^K p_j} = \frac{K+1}{2}~.
\]
\end{fact}
We are now ready to introduce and analyze the new algorithm Exp3-DOM for the adversarial, informed
and directed setting. Exp3-DOM (see Algorithm \ref{a:exp3dom}) runs $\mathcal{O}(\log K)$ variants of Exp3 
indexed by $b = 0,1,\dots,\lfloor \log_2 K\rfloor$. At time $t$ the algorithm is given observation system 
$\{S_{i,t}\}_{i\in V}$,
and computes a dominating set $R_t$ of the directed graph $G_t$ induced by $\{S_{i,t}\}_{i\in V}$.
Based on the size $|R_t|$ of $R_t$, the algorithm uses instance $b_t = \lfloor \log_2|R_t|\rfloor$
to pick action $I_t$. We use a superscript $b$ to denote the quantities relevant to the variant of Exp3 indexed by $b$. Similarly to the analysis of Exp3-SET, the key quantities are
\[
    q^{(b)}_{i,t} = \sum_{j \,:\, i \in S_{j,t}} p^{(b)}_{j,t} = \sum_{j \,:\, j \reach{t} i} p^{(b)}_{j,t}
\qquad\text{and}\qquad
    Q^{(b)}_t = \sum_{i \in V} \frac{p^{(b)}_{i,t}}{q^{(b)}_{i,t}}~,\qquad b = 0, 1, \ldots, \lfloor \log_2 K\rfloor~.
\]
Let $\Tb = \bigl\{ t=1,\dots,T \,:\, |R_t| \in [2^b,2^{b+1}-1] \bigr\}$. Clearly, the sets $\Tb$ are a
partition of the time steps $\{1,\dots,T\}$, so that $\sum_b |\Tb| = T$. Since the adversary adaptively chooses the dominating sets $R_t$, the sets $\Tb$ are random.
This causes a problem in tuning the parameters $\gammab$.
For this reason, we do not prove a regret bound for Exp3-DOM, where each instance uses a fixed $\gammab$,
but for a slight variant (described in the proof of Theorem~\ref{thm:alg} ---see the appendix)
where each $\gammab$ is set through a doubling trick.
%
\begin{algorithm2e}[t]
\SetKwSty{textrm} \SetKwFor{For}{{\bf For}}{}{}
\SetKwIF{If}{ElseIf}{Else}{if}{}{else if}{else}{}
\SetKwFor{While}{while}{}{}
{\bf Input:} Exploration parameters $\gammab \in (0,1]$ for $b \in \bigl\{0,1,\ldots, \lfloor \log_2 K\rfloor\bigr\}$;\\
{\bf Initialization:} $w^{(b)}_{i,1} = 1$ for all $i \in V$ and $b \in \bigl\{0,1,\ldots, \lfloor \log_2 K\rfloor\bigr\}$;\\
\For{$t=1,2,\dots$ :}
{ {
\begin{enumerate}
\item Observation system $\{S_{i,t}\}_{i\in V}$ is generated {\em and disclosed}~;
\item Compute a dominating set $R_t\subseteq V$ for $G_t$ associated with $\{S_{i,t}\}_{i\in V}$~;
\item Let $b_t$ be such that $|R_t| \in \bigl[2^{b_t},2^{b_t+1}-1\bigr]$;
\item Set $W^{(b_t)}_t = \sum_{i \in V} w^{(b_t)}_{i,t};$
\item Set ${\dt p^{(b_t)}_{i,t} = \bigl(1-\gammabt\bigr) \frac{w^{(b_t)}_{i,t}}{W^{(b_t)}_{t}} + \frac{\gammabt}{|R_t|} \Ind{i\in R_t} }$;
\item Play action $I_t$ drawn according to distribution $p^{(b_t)}_t = \bigl(p^{(b_t)}_{1,t},\dots,p^{(b_t)}_{V,t}\bigr)$~;
\item Observe pairs $(i,\loss_{i,t})$ for all $i \in S_{I_t,t}$;
\item For any $i \in V$ set $w^{(b_t)}_{i,t+1} = w^{(b_t)}_{i,t}\,\exp\bigl(-\gammabt\,\hloss^{(b_t)}_{i,t}/2^{b_t}\bigr)$,
where
\[
    \hloss^{(b_t)}_{i,t}
=
    \frac{\loss_{i,t}}{q^{(b_t)}_{i,t}}\,\Ind{i \in S_{I_t,t}}
\qquad\text{and}\qquad
    q^{(b_t)}_{i,t} = \sum_{j \,:\, j \reach{t} i} p^{(b_t)}_{j,t}~.
\]
\end{enumerate}
} }
\vspace{-0.2in}
\caption{Exp3-DOM}
\label{a:exp3dom}
\end{algorithm2e}
%
\begin{theorem}\label{thm:alg}
In the adversarial and directed case, the regret of Exp3-DOM satisfies
\begin{equation}
\label{eq:gammabfixed}
    \max_{k\in V} \E\bigl[L_{A,T} - L_{k,T}\bigr]
\le
    \sum_{b=0}^{\lfloor\log_2 K\rfloor} \left( \frac{2^b\ln K}{\gammab}
    +  \gammab\E\left[\sum_{t \in \Tb} \left(1 + \frac{Q^{(b)}_t}{2^{b+1}}\right)\right]\right)~.
\end{equation}
Moreover, if we use a doubling trick to choose $\gammab$ for each $b = 0,\dots,\lfloor \log_2 K\rfloor$, then
\begin{equation}
\label{eq:doublingtrick}
  \max_{k\in V} \E\bigl[L_{A,T} - L_{k,T}\bigr]
=
{\mathcal O}\left((\ln K)\,\E\left[\sqrt{\sum_{t=1}^T \left(4|R_t|
+ Q_t^{(b_t)}\right)}\right] + (\ln K) \ln(KT)\right)~.
\end{equation}
\end{theorem}
Importantly, the next result shows how bound~(\ref{eq:doublingtrick}) of Theorem~\ref{thm:alg} can be
expressed in terms of the sequence $\alpha(G_t)$ of independence numbers of graphs $G_t$ whenever
the Greedy Set Cover algorithm~\cite{Chv79} (see Section \ref{s:prel}) is used to compute the dominating set $R_t$
of the observation system at time $t$.
\begin{corollary}\label{c:final}
If Step~2 of Exp3-DOM uses the Greedy Set Cover algorithm to compute the dominating sets $R_t$,
then the regret of Exp-DOM with doubling trick satisfies
\[
   \max_{k\in V} \E\bigl[L_{A,T} - L_{k,T}\bigr]
=
    \mathcal{O}\left(\ln(K)\sqrt{\ln(KT)\sum_{t=1}^T \alpha(G_t)} + \ln(K)\ln(KT) \right)~,
\]
where, for each $t$, $\alpha(G_t)$ is the independence number of the graph $G_t$ induced by observation system
$\{S_{i,t}\}_{i\in V}$.
\end{corollary}
%

\section{Conclusions and work in progress}
We have investigated online prediction problems in partial information regimes that interpolate between the classical bandit and expert settings. We have shown a number of results characterizing prediction performance in terms of: the structure of the observation system, the amount of information available before prediction, the nature (adversarial or fully random) of the process generating the observation system.
Our results are substantial improvements over the paper~\cite{MS11} that initiated this interesting line of research.
Our improvements are diverse, and range from considering both informed and uninformed settings to delivering more refined graph-theoretic characterizations, from providing more efficient algorithmic solutions to relying on simpler (and often more general) analytical tools.

Some research directions we are currently pursuing are the following. 
\begin{enumerate}
\item We are currently investigating the extent to which our results could be applied to the case when the observation system $\{S_{i,t}\}_{i\in V}$ may depend on the loss $\ell_{I_t,t}$ of player's action $I_t$. Notice that this would prevent a direct construction of an unbiased estimator
for unobserved losses, which many worst-case bandit algorithms (including ours ---see the appendix) hinge upon.
\item The upper bound contained in Corollary~\ref{c:ndag} and expressed in terms of $\mas(\cdot)$ is almost certainly 
suboptimal, even in the uninformed setting, and we are trying to see if more adequate graph complexity measures can be used instead.
\item Our lower bound (Corollary~\ref{c:lower}) heavily relies on the corresponding lower bound in~\cite{MS11} which, in turn, refers to a constant graph sequence. We would like to provide a more complete charecterization applying to sequences of adversarially-generated graphs $G_1, G_2, \ldots, G_T$ in terms of sequences of their corresponding independence numbers 
$\alpha(G_1), \alpha(G_2), \ldots, \alpha(G_T)$ (or variants thereof), in both the uninformed and the informed settings.
\end{enumerate}

\subsubsection*{Acknowledgments}
The first author was supported in part by an ERC advanced grant, by a USA-Israeli BSF grant, and by the Israeli I-CORE program. The second author acknowledges partial support by MIUR (project ARS TechnoMedia, PRIN 2010-2011, grant no.\ 2010N5K7EB\_003). The fourth author was supported in part by a grant from the Israel Science Foundation, a grant from the United States-Israel Binational Science Foundation (BSF), a grant by Israel Ministry of Science and Technology and the Israeli Centers of Research Excellence (I-CORE) program (Center No.\ 4/11).

\newpage

\appendix

\section{Technical lemmas and proofs}\label{s:appendix}
This section contains the proofs of all technical results occurring in the main text, along with
ancillary graph-theoretic lemmas. Throughout this appendix, $\E_t[\cdot]$ is a shorthand for
$\E\bigl[\cdot\mid I_1,\dots,I_{t-1}\bigr]$.

\medskip\noindent
\begin{proofof}{Theorem~\ref{thm:noexp}}
Following the proof of Exp3~\cite{AuerCeFrSc02}, we have
\begin{align*}
\frac{W_{t+1}}{W_t}
&= \sum_{i \in V} \frac{w_{i,t+1}}{W_t}\\
&= \sum_{i \in V} \frac{w_{i,t}\,\exp(-\eta\,\hloss_{i,t})}{W_t}\\
&= \sum_{i \in V} p_{i,t}\,\exp(-\eta\,\hloss_{i,t})\\
&\leq \sum_{i \in V} p_{i,t}\,\left(1 - \eta\hloss_{i,t} + \frac{1}{2}\,\eta^2(\hloss_{i,t})^2\right) \quad \text{using $e^{-x} \leq 1-x+x^2/2$ for all $x \ge 0$}\\
&\leq 1 - \eta\,\sum_{i \in V} p_{i,t}\hloss_{i,t} + \frac{\eta^2}{2}\,\sum_{i \in V} p_{i,t}(\hloss_{i,t})^2~.
\end{align*}
Taking logs, using $\ln(1-x) \le -x$ for all $x \ge 0$, and summing
over $t = 1, \ldots, T$ yields
\[
\ln \frac{W_{T+1}}{W_1} \leq -\eta\,\sum_{t=1}^T \sum_{i \in V} p_{i,t}\hloss_{i,t} +
\frac{\eta^2}{2}\,\sum_{t=1}^T \sum_{i \in V} p_{i,t}(\hloss_{i,t})^2~.
\]
Moreover, for any fixed comparison action $k$, we also have
\[
\ln \frac{W_{T+1}}{W_1} \geq \ln \frac{w_{k,T+1}}{W_1} = -\eta\,\sum_{t=1}^T \hloss_{k,t} - \ln K~.
\]
Putting together and rearranging gives
\begin{equation}\label{e:eq1}
    \sum_{t=1}^T \sum_{i \in V} p_{i,t}\hloss_{i,t}
\le
    \sum_{t=1}^T \hloss_{k,t} + \frac{\ln K}{\eta}
    + \frac{\eta}{2}\,\sum_{t=1}^T \sum_{i \in V} p_{i,t}(\hloss_{i,t})^2~.
\end{equation}
Note that, for all $i \in V$,
\[
\E_t[\hloss_{i,t}] = \sum_{j\,:\, i \in S_{j,t}}
p_{j,t}\,\frac{\loss_{i,t}}{q_{i,t}}
                = \sum_{j \,:\, j \reach{t} i} p_{j,t}\,\frac{\loss_{i,t}}{q_{i,t}}
                = \frac{\loss_{i,t}}{q_{i,t}}\sum_{j \,:\, j \reach{t} i} p_{j,t}
                = \loss_{i,t}~.
\]
Moreover,
\[
\E_t\bigl[(\hloss_{i,t})^2\bigr] = \sum_{j\,:\, i \in S_{j,t}}
p_{j,t}\,\frac{\loss^2_{i,t}}{q^2_{i,t}}
                = \frac{\loss^2_{i,t}}{q^2_{i,t}}\sum_{j \,:\, j \reach{t} i} p_{j,t}
                \leq \frac{1}{q^2_{i,t}}\sum_{j \,:\, j \reach{t} i} p_{j,t}
                = \frac{1}{q_{i,t}}~.
\]
Hence, taking expectations $\E_t$ on both sides of~(\ref{e:eq1}), and recalling the definition
of $Q_t$, we can write
\begin{equation}\label{e:conditionalregret}
    \sum_{t=1}^T \sum_{i \in V} p_{i,t}\,\loss_{i,t}
\le
    \sum_{t=1}^T \loss_{k,t} + \frac{\ln K}{\eta} + \frac{\eta}{2}\,\sum_{t=1}^T Q_t~.
\end{equation}
Finally, taking expectations to remove conditioning gives
\begin{equation*}
\E\bigl[L_{A,T} - L_{k,T}\bigr]
\le
    \frac{\ln K}{\eta} + \frac{\eta}{2}\sum_{t=1}^T \E[Q_t]~,
\end{equation*}
as claimed.
\end{proofof}
\begin{proofof}{Corollary~\ref{thm:random_er}}
Fix round $t$, and let $G = (V,D)$ be the Erd\H{o}s-Renyi random graph generated at time $t$,
$N_i^-$ be the in-neighborhood of node $i$, i.e., the set of
nodes $j$ such that $(j,i) \in D$, and denote by $d^-_i$ the indegree of $i$.

\begin{claim}\label{cl:1}
Let $p_1, \ldots, p_K$ be an arbitrary probability distribution defined over $V$,
$f: V\rightarrow V$ be an arbitrary permutation of $V$, and
${\E_f}$ denote the expectation w.r.t. permutation $f$ when $f$ is drawn uniformly
at random. Then, for any $i \in V$,
we have
\[
\E_f\left[\frac{p_{f(i)}}
             {p_{f(i)} + \sum_{j\,:\, f(j) \in N^-_{f(i)}} p_{f(j)}}\right]
=
\frac{1}{1+d^-_i}~.
\]
\end{claim}
\begin{proofof}{Claim~\ref{cl:1}}
Consider selecting a subset $S\subset V$ of $1+d^-_i$ nodes. We shall
consider the contribution to the expectation when $S=N^-_{f(i)}\cup\{f(i)\}$.
Since there are $K(K-1)\cdots(K-d^-_i+1)$ terms (out of $K!$) contributing to the expectation,
we can write
\begin{eqnarray*}
\E_f\left[\frac{p_{f(i)}}{p_{f(i)} + \sum_{j\,:\, f(j) \in N^-_{f(i)}} p_{f(j)}}\right]
&=&
\frac{1}{\binom{K}{d^-_i}}
\sum_{S\subset V, |S|=d^-_i} \frac{1}{1+d^-_i}\sum_{i\in S} \frac{p_i}{p_i+\sum_{j\in S,j\neq i} p_j}\\
&=&
\frac{1}{\binom{K}{d^-_i}}
\sum_{S\subset V, |S|=d^-_i} \frac{1}{1+d^-_i}\\
&=& \frac{1}{1+d^-_i}~.
\end{eqnarray*}
\end{proofof}
\begin{claim}\label{cl:2}
Let $p_1, \ldots, p_K$ be an arbitrary probability distribution defined over $V$, and
${\E}$ denote the expectation w.r.t.\ the Erd\H{o}s-Renyi random draw of arcs at time $t$.
Then, for any fixed $i \in V$, we have
\[
\E\left[\frac{p_i}{p_i + \sum_{j\,:\,j \reach{t}i} p_j}\right]
=
\frac{1}{rK}\left(1-(1-r)^K\right)~.
\]
\end{claim}
\begin{proofof}{Claim~\ref{cl:2}}
For the given $i \in V$ and time $t$, consider the Bernoulli random
variables $X_{j}, j\in V\setminus\{i\}$, and denote by
$\E_{j\,:\,j\neq i}$ the expectation w.r.t.\ all of them. We
symmetrize $\E\left[\frac{p_i}{p_i + \sum_{j\,:\,j \reach{t}i}
p_j}\right]$ by means of a random permutation $f$, as in Claim~\ref{cl:1}. We can write
\begin{eqnarray*}
\E\left[\frac{p_i}{p_i + \sum_{j\,:\,j \reach{t}i} p_j}\right]
&= &
\E_{j\,:\,j\neq i}\left[\frac{p_i}{p_i + \sum_{j\,:\,j\neq i} X_j p_j}\right]\\
&= &
\E_{j\,:\,j\neq i} \E_f \left[\frac{p_{f(i)}}{p_{f(i)} + \sum_{j\,:\,j\neq i} X_{f(j)} p_{f(j)}}\right]
        \qquad ({\mbox{by symmetry}})\\
&= &
\E_{j\,:\,j\neq i} \left[\frac{1}{1 + \sum_{j\,:\,j\neq i} X_j}\right]
        \qquad ({\mbox{from Claim \ref{cl:1}}})\\
&=&
\sum_{i=0}^{K-1} \binom{K-1}{i} r^i (1-r)^{K-1-i} \frac{1}{i+1}\\
&=&
\frac{1}{rK} \sum_{i=0}^{K-1} \binom{K}{i+1} r^{i+1}(1-r)^{K-1-i} \\
&=&
\frac{1}{rK} \left( 1 - (1-r)^K\right)~.
\end{eqnarray*}
\end{proofof}
At this point, we follow the proof of Theorem~\ref{thm:noexp} up until~(\ref{e:conditionalregret}).
We take an expectation $\E_{G_1, \ldots, G_T}$
w.r.t.\ the randomness in generating the sequence of graphs $G_1, \ldots, G_T$.
This yields
\[
    \sum_{t=1}^T \E_{G_1, \ldots, G_T}\left[\sum_{i \in V} p_{i,t}\,\loss_{i,t}\right]
\le
    \sum_{t=1}^T \loss_{k,t} + \frac{\ln K}{\eta} + \frac{\eta}{2}\,\sum_{t=1}^T \E_{G_1, \ldots, G_T}\left[Q_t\right]~.
\]
We use Claim \ref{cl:2} to upper bound $\E_{G_1, \ldots, G_T}\left[Q_t\right]$ by
$\frac{1}{r} \left( 1 - (1-r)^K\right)$, and take the outer expectation to remove conditioning, as in
the proof of Theorem~\ref{thm:noexp}. This concludes the proof.
\end{proofof}
The following lemma can be seen as a generalization of Lemma 3 in~\cite{MS11}.
\begin{lemma}\label{lemma:nDGA}
Let $G = (V,D)$ be a directed graph with vertex set 
$V = \{1,\ldots,K\}$, and arc set $D$. Let $N_i^-$ be the in-neighborhood
of node $i$, i.e., the set of nodes $j$ such that $(j,i) \in D$.
Then
\[
\sum_{i=1}^K \frac{p_i}{p_i+ \sum_{j \in N_i^-}\ p_j} \leq \mas(G)~.
\]
\end{lemma}
\begin{proof}
We will show that there is a subset of vertices $V'$ such that the
induced graph is acyclic and $|V'|\geq \sum_{i=1}^K \frac{p_i}{p_i+ \sum_{j \in N_i^-}\ p_j}$.

We prove the lemma by growing set $V'$ starting off from $V' = \emptyset$.
Let
\[
\Phi_0=\sum_{i=1}^K \frac{p_i}{p_i+ \sum_{j \in N_i^-}\ p_j}~,
\]
and $i_1$ be the
vertex which minimizes $p_{i}+ \sum_{j \in N_{i}^-}\ p_j$ over $i \in V$.
We are going to delete $i_1$ from the graph, along with all its incoming neighbors 
(set $N_{i_1}^-$), and all edges which are incident (both departing and incoming) 
to these nodes, and then iterating on the remaining graph. Let us denote the in-neighborhoods 
of the shrunken graph from the first step by $N_{i,1}^-$.

The contribution of all the deleted vertices to $\Phi_0$ is
\[
\sum_{r\in N_{i_1}^-\cup \{i_1\}} \frac{p_r}{p_r+ \sum_{j \in
N_r^-}\ p_j} \leq \sum_{r\in N_{i_1}^-\cup \{i_1\}}
\frac{p_r}{p_{i_1}+ \sum_{j \in N_{i_1}^-}\ p_j}=1~,
\]
where the inequality follows from the minimality of $i_1$.

Let $V' \leftarrow V'\cup\{i_1\}$, and  $V_1= V-(N_{i_1}^-\cup \{i_1\})$. Then
from the first step we have
\[
\Phi_1=\sum_{i\in V_1} \frac{p_i}{p_i+ \sum_{j \in N_{i,1}^-}\ p_j} 
\geq
\sum_{i\in V_1} \frac{p_i}{p_i+ \sum_{j \in N_{i}^-}\ p_j}
\geq
\Phi_0 -1~.
\]
We apply the very same argument to $\Phi_1$ with node $i_2$ (minimizing
$p_i+ \sum_{j \in N_{i,1}^-}\ p_j$ over $i \in V_1$), to $\Phi_2$ with node $i_3$, \ldots, to $\Phi_{s-1}$
with node $i_s$, up until 
$\Phi_s = 0$, i.e., up until no nodes are left in the shrunken graph.
This gives $\Phi_0 \leq s = |V'|$, where $V' = \{i_1, i_2, \ldots, i_s\}$. 
Moreover, since in each step $r = 1, \ldots, s$ we remore
all remaining arcs incoming to $i_r$, the graph induced by set $V'$ cannot contain cycles.
%
%
\end{proof}
\begin{proofof}{Corollary~\ref{c:ndag}}
The claim follows from a direct combination of Theorem~\ref{thm:noexp} with Lemma~\ref{lemma:nDGA}.
\end{proofof}
\begin{proofof}{Fact~\ref{l:bad}}
Using standard properties of geometric sums, one can immediately see that
\[
\sum_{i=1}^{K}\frac{p_i}{\sum_{j=i}^K p_j} = \sum_{i=1}^{K-1}
\frac{2^{-i}}{2^{-i+1}} + \frac{2^{-K+1}}{2^{-K+1}}= \frac{K-1}{2} +1= \frac{K+1}{2}~,
\]
hence the claimed result.
\end{proofof}
The following graph-theoretic lemma turns out to be fairly useful for analyzing directed settings.
It is a directed-graph counterpart to a well-known result~\cite{c79,w81} holding
for undirected graphs. 
%
\begin{lemma}\label{l:amlemma}
Let $G = (V,D)$ be a directed graph, 
with $V = \{1,\ldots,K\}$.
Let $d_i^-$ be the indegree of node $i$, and $\alpha = \alpha(G)$ be the independence number of $G$. Then
\[
\sum_{i=1}^K \frac{1}{1+d_i^-} \leq 2\alpha\,\ln\left(1+\frac{K}{\alpha}\right)~.
\]
\end{lemma}
\begin{proof}
We will proceed by induction, starting off from the original $K$-node graph $G = G_K$ with indegrees
$\{d_{i}^-\}_{i=1}^K = \{d_{i,K}^-\}_{i=1}^K$, and
independence number $\alpha = \alpha_K$, and then progressively shrink $G$ by eliminating nodes and incident
(both departing and incoming) arcs, thereby obtaining a sequence of smaller and smaller graphs
$G_K, G_{K-1}, G_{K-2}, \ldots $, and associated
indegrees $\{d_{i,K}^-\}_{i=1}^{K}$, $\{d_{i,K-1}^-\}_{i=1}^{K-1}$, $\{d_{i,K-2}^-\}_{i=1}^{K-2}$, \ldots, and
independence numbers $\alpha_K, \alpha_{K-1}, \alpha_{K-2}, \ldots$. Specifically, in step $s$ we sort nodes
$i = 1, \ldots, s$ of $G_s$ in nonincreasing value of $d_{i,s}^-$, and obtain $G_{s-1}$ from $G_s$
by eliminating node $1$ (i.e., one having the largest indegree among the nodes of $G_s$), along with its incident
arcs. On all such graphs, we will use the classical
Turan's theorem~(e.g., \cite{as04}) stating that any {\em undirected} graph with $n_s$ nodes and
$m_s$ edges has an independent set of size at least $\frac{n_s}{\frac{2m_s}{n_s}+1}$. This implies
that if $G_s = (V_s,D_s)$, then $\alpha_s$ satisfies\footnote
{
Notice that $|D_s|$ is at least as large as the number of edges of the undirected version of
$G_s$ which the independence number $\alpha_s$ actually refers to.
}
\begin{equation}\label{e:turan}
\frac{|D_s|}{|V_s|} \geq \frac{|V_s|}{2\alpha_s} - \frac{1}{2}~.
\end{equation}
We then start from $G_K$. We can write
\[
d_{1,K}^- = \max_{i=1\ldots K} d_{i,K}^- \geq \frac{1}{K}\,\sum_{i=1}^K d_{i,K}^- = \frac{|D_K|}{|V_K|}
\geq \frac{|V_K|}{2\alpha_K} - \frac{1}{2}~.
\]
%
Hence,
\begin{eqnarray*}
\sum_{i=1}^K \frac{1}{1+d_{i,K}^-}
&=& \frac{1}{1+d_{1,K}^-} + \sum_{i=2}^K \frac{1}{1+d_{i,K}^-} \\
&\leq&
\frac{2\alpha_K}{\alpha_K+K} + \sum_{i=2}^K \frac{1}{1+d_{i,K}^-}\\
&\leq&
\frac{2\alpha_K}{\alpha_K+K} + \sum_{i=1}^{K-1} \frac{1}{1+d_{i,K-1}^-},
\end{eqnarray*}
where the last inequality follows from $d_{i+1,K}^- \geq d_{i,K-1}^-$, $i = 1, \ldots K-1$, due to the
arc elimination turning $G_K$ into $G_{K-1}$. Recursively applying the very same argument to $G_{K-1}$
(i.e., to the sum $\sum_{i=1}^{K-1} \frac{1}{1+d_{i,K-1}^-}$),
and then iterating all the way to $G_1$ yields the upper bound
\[
\sum_{i=1}^K \frac{1}{1+d_{i,K}^-} \leq \sum_{i=1}^K \frac{2\alpha_i}{\alpha_i+i}~.
\]
Combining with $\alpha_i \leq \alpha_K = \alpha$, and
$\sum_{i=1}^K \frac{1}{\alpha+i} \leq \ln \left(1+\frac{K}{\alpha} \right)$
concludes the proof.
\end{proof}

The next lemma 
relates the size $|R_t|$ of the dominating set $R_t$ computed by the Greedy
Set Cover algorithm of~\cite{Chv79} operating on the time-$t$ observation system $\{S_{i,t}\}_{i\in V}$
to the independence number $\alpha(G_t)$ and the domination number $\gamma(G_t)$ of $G_t$.
\begin{lemma}\label{l:greedycover}
Let $\{S_i\}_{i \in V}$ be an observation system, and $G = (V,D)$ be the induced directed graph,
with vertex set $V = \{1,\ldots,K\}$,
independence number $\alpha = \alpha(G)$, and domination number $\gamma = \gamma(G)$.
Then the dominating set $R$ constructed by the Greedy Set Cover algorithm (see Section~\ref{s:prel})
satisfies
\[
|R| \le \min\bigl\{ \gamma(1+\ln K), \lceil 2\alpha\ln K \rceil + 1\bigr\}~.
\]
\end{lemma}
\begin{proof}
As recalled in Section~\ref{s:prel}, the Greedy Set Cover algorithm of~\cite{Chv79} achieves $|R| \le \gamma(1+\ln K)$.
In order to prove the other bound, consider the sequence of graphs $G = G_1, G_2,\dots$, where each
$G_{s+1} = (V_{s+1},D_{s+1})$ is obtained by removing from $G_s$ the vertex $i_s$ selected by the Greedy Set Cover
algorithm, together with all the vertices in $G_{s}$ that are dominated by $i_s$, and all arcs incident to these vertices.
By definition of the algorithm, the outdegree $d_s^+$ of $i_s$ in $G_s$ is largest in $G_s$. Hence,
\[
    d_s^+ \ge \frac{|D_s|}{|V_s|} \ge \frac{|V_s|}{2\alpha_s} - \frac{1}{2} \ge \frac{|V_s|}{2\alpha} - \frac{1}{2}
\]
by Turan's theorem (e.g.,~\cite{as04}),
where $\alpha_s$ is the independence number of $G_s$ and $\alpha\ge \alpha_s$. This shows that
\[
   |V_{s+1}| = |V_s| - d_s^+ - 1 \le |V_s|\left(1 - \frac{1}{2\alpha}\right) \le |V_s|e^{-1/(2\alpha)}~.
\]
Iterating, we obtain $|V_s| \le K\,e^{-s/(2\alpha)}$.
Choosing $s = \lceil 2\alpha\ln K \rceil+1$ gives $|V_s| < 1$, thereby covering all nodes.
Hence the dominating set $R = \{i_1, \ldots, i_s\}$ so constructed
satisfies $|R| \leq \lceil 2\alpha\ln K \rceil+1$.
\end{proof}
\begin{lemma}\label{l:ancillary}
If $a, b \geq 0$, and $a+b \geq B > A > 0$, then
\[
\frac{a}{a+b-A} \leq \frac{a}{a+b} + \frac{A}{B-A}~.
\]
\end{lemma}
\begin{proof}
\[
\frac{a}{a+b-A} - \frac{a}{a+b} = \frac{aA}{(a+b)(a+b-A)} \leq \frac{A}{a+b-A} \leq \frac{A}{B-A}~.
\]
\end{proof}
We now lift Lemma~\ref{l:amlemma} to a more general statement.
\begin{lemma}\label{l:weightedamlemma}
Let $G = (V,D)$ be a directed graph, 
with vertex set $V = \{1,\ldots,K\}$,
and arc set $D$. Let $N_i^-$ be the in-neighborhood of node $i$, i.e., the set of nodes $j$ such that
$(j,i) \in D$.
Let
$\alpha$ be the independence number of $G$, $R \subseteq V$ be a dominating set for $G$ of size $r=|R|$,
and $p_1, \ldots, p_K$
be a probability distribution defined over $V$, such that $p_i \geq \beta > 0$, for $i \in R$.
Then
\[
\sum_{i=1}^K \frac{p_i}{p_i+ \sum_{j \in N_i^-}\ p_j}
\leq
2\alpha\,\ln\left(1+\frac{\lceil\frac{K^2}{r\beta}\rceil+K}{\alpha}\right) + 2r~.
\]
\end{lemma}
\begin{proof}
The idea is to appropriately discretize the probability values $p_i$, and then upper bound
the discretized counterpart of $\sum_{i=1}^K \frac{p_i}{p_i+ \sum_{j \in N_i^-}\ p_j}$ by reducing to
an expression that can be handled by Lemma \ref{l:amlemma}. In order to make this discretization
effective, we need to single out the terms $\frac{p_i}{p_i+ \sum_{j \in N_i^-}\ p_j}$
corresponding to nodes $i \in R$. We first write
\begin{eqnarray}
\sum_{i=1}^K \frac{p_i}{p_i+ \sum_{j \in N_i^-}\ p_j}
&=&
\sum_{i\in R} \frac{p_i}{p_i+ \sum_{j \in N_i^-}\ p_j} + \sum_{i\notin R} \frac{p_i}{p_i+ \sum_{j \in N_i^-}\ p_j}\nonumber\\
&\leq& 
r + \sum_{i\notin R} \frac{p_i}{p_i+ \sum_{j \in N_i^-}\ p_j} \,,\label{e:prelim}
\end{eqnarray}
and then focus on (\ref{e:prelim}).

Let us discretize the unit interval\footnote
{
The zero value won't be of our concern here, because
if $p_i = 0$, the corresponding term in (\ref{e:prelim}) can be disregarded.
}
$(0,1]$ into subintervals $(\frac{j-1}{M},\frac{j}{M}]$, $j = 1, \ldots, M$,
where $M = \lceil\frac{K^2}{r\beta}\rceil$.
Let $\hp_i = j/M$ be the discretized version of $p_i$,
being $j$ the unique integer such that
\[
\hp_i - 1/M < p_i \leq \hp_i~.
\]
%
Let us focus on a single node $i\notin R$ with indegree $d_i^- = |N_i^-|$, and
introduce the shorthand notation $P_i = \sum_{j \in N_i^-}\ p_j$,
and $\hP_i = \sum_{j \in N_i^-}\ \hp_j$.
We have that $\hP_i \geq P_i \geq \beta$, since $i$ is dominated by some node $j \in R \cap N_i^-$
such that $p_j \geq \beta$. Moreover,
$P_i > \hP_i - \frac{d_i^-}{M} \geq \beta - \frac{d_i^-}{M} >0$,
and $\hp_i+\hP_i \geq \beta$.
Hence, for any fixed node $i\notin R$, we can write
\begin{eqnarray*}
\frac{p_i}{p_i+ P_i}
&\leq&
 \frac{\hp_i}{\hp_i+ P_i}\\
&<&
 \frac{\hp_i}{\hp_i+ \hP_i -  \frac{d_i^-}{M}}\\
&\leq&
 \frac{\hp_i}{\hp_i+ \hP_i} + \frac{d_i^-/M}{\beta - d_i^-/M}\\
&=&
 \frac{\hp_i}{\hp_i+ \hP_i} + \frac{d_i^-}{\beta M - d_i^-}\\
&< &
 \frac{\hp_i}{\hp_i+ \hP_i} + \frac{r}{K-r},
\end{eqnarray*}
where in the second-last inequality we used Lemma \ref{l:ancillary} with $a = \hp_i$, $b= \hP_i$, $A = d_i^-/M$,
and $B = \beta > d_i^-/M$. Recalling (\ref{e:prelim}), and summing over $i$ then gives
\begin{equation}\label{e:discretization}
\sum_{i=1}^K \frac{p_i}{p_i+ P_i} \leq r + \sum_{i \notin R} \frac{\hp_i}{\hp_i + \hP_i} + r =
\sum_{i \notin R} \frac{\hp_i}{\hp_i + \hP_i} + 2r~.
\end{equation}
Therefore, we continue by bounding from above the right-hand side of (\ref{e:discretization}). We first observe
that
\begin{equation}\label{e:discretization2}
\sum_{i\notin R} \frac{\hp_i}{\hp_i + \hP_i} = \sum_{i\notin R} \frac{\hs_i}{\hs_i + \hS_i},\qquad
\hS_i = \sum_{j \in N_i^-} \hs_j~,
\end{equation}
where $\hs_i = M\hp_i$, $i = 1, \ldots, K$, are integers. Based on the original graph $G$, we construct a new graph
$\hG$ made up of connected cliques. In  particular:
\begin{itemize}
\item Each node $i$ of $G$ is replaced in $\hG$ by a clique $C_i$ of size $\hs_i$; nodes within $C_i$ are
connected by length-two cycles.
\item If arc $(i,j)$ is in $G$, then for {\em each} node of $C_i$ draw an arc towards {\em each}
node of $C_j$.
\end{itemize}
We would like to apply Lemma \ref{l:amlemma} to $\hG$. Notice that, by the above construction:
\begin{itemize}
\item The independence number of $\hG$ is the same as that of $G$;
\item The indegree $\hd_k^-$ of each node $k$ in clique $C_i$ satisfies  $\hd_k^- = \hs_i-1 + \hS_i$.
\item The total number of nodes of $\hG$ is
\[
\sum_{i=1}^K \hs_i = M\sum_{i=1}^K \hp_i < M\sum_{i=1}^K \left(p_i + \frac{1}{M}\right) = M+K~.
\]
\end{itemize}
Hence, we are in a position to apply Lemma \ref{l:amlemma} to $\hG$ with indegrees $\hd_k^-$, revealing that
\[
\sum_{i\notin R} \frac{\hs_i}{\hs_i + \hS_i} =
\sum_{i\notin R} \sum_{k \in C_i} \frac{1}{1+\hd_k^-}
\leq \sum_{i=1}^K \sum_{k \in C_i} \frac{1}{1+\hd_k^-}
\leq 2\alpha\ln\left(1+\frac{M+K}{\alpha}\right)~.
\]
Putting together as in (\ref{e:discretization}) and (\ref{e:discretization2}), and recalling the
value of $M$ gives the claimed result.
\end{proof}
\begin{proofof}{Theorem~\ref{thm:alg}}
We start to bound the contribution to the overall regret of an instance indexed by $b$. When clear from the context, we remove the superscript $b$ from $\gammab$, $w^{(b)}_{i,t}$, $p^{(b)}_{i,t}$, and other related quantities. For any $t\in\Tb$ we have
\begin{align*}
    \frac{W_{t+1}}{W_t}
&=
    \sum_{i \in V} \frac{w_{i,t+1}}{W_t}
\\&=
    \sum_{i \in V} \frac{w_{i,t}}{W_t}\,\exp\bigl(-(\gamma/2^b)\,\hloss_{i,t}\bigr)
\\&=
    \sum_{i \in R_t} \frac{p_{i,t}-\gamma/|R_t|}{1-\gamma}\,\exp\bigl(-(\gamma/2^b)\,\hloss_{i,t}\bigr) + \sum_{i \not\in R_t} \frac{p_{i,t}}{1-\gamma}\,\exp\bigl(-(\gamma/2^b)\,\hloss_{i,t}\bigr)
\\ &\le
    \sum_{i \in R_t} \frac{p_{i,t}-\gamma/|R_t|}{1-\gamma}\,
    \left(1 - \frac{\gamma}{2^b}\hloss_{i,t} + \frac{1}{2}\left(\frac{\gamma}{2^b}\hloss_{i,t}\right)^2\right)
    + \sum_{i \not\in R_t} \frac{p_{i,t}}{1-\gamma}\,\left(1 - \frac{\gamma}{2^b}\hloss_{i,t} + \frac{1}{2}\left(\frac{\gamma}{2^b}\hloss_{i,t}\right)^2\right)\\
& \text{(using $e^{-x} \leq 1-x+x^2/2$ for all $x \ge 0$)}
\\ &\le
    1 - \frac{\gamma/2^b}{1-\gamma}\sum_{i \in V} p_{i,t}\hloss_{i,t}
    + \frac{\gamma^2/2^b}{1-\gamma}\sum_{i \in R_t} \frac{\hloss_{i,t}}{|R_t|}
    + \frac{1}{2}\frac{(\gamma/2^b)^2}{1-\gamma}\sum_{i \in V} p_{i,t}\bigl(\hloss_{i,t}\bigr)^2~.
\end{align*}
%
%
Taking logs, upper bounding, and summing over $t \in \Tb$ yields
\[
    \ln\frac{W_{|\Tb|+1}}{W_1}
\le
    - \frac{\gamma/2^b}{1-\gamma}\sum_{t \in \Tb} \sum_{i \in V} p_{i,t}\hloss_{i,t}
    + \frac{\gamma^2/2^b}{1-\gamma}\sum_{t \in \Tb} \sum_{i \in R_t} \frac{\hloss_{i,t}}{|R_t|}
    + \frac{1}{2}\frac{(\gamma/2^b)^2}{1-\gamma} \sum_{t \in \Tb}\sum_{i \in V} p_{i,t}\bigl(\hloss_{i,t}\bigr)^2~.
\]
Moreover, for any fixed comparison action $k$, we also have
\[
    \ln\frac{W_{|\Tb|+1}}{W_1}
\ge
    \ln\frac{w_{k,|\Tb|+1}}{W_1} = -\frac{\gamma}{2^b}\sum_{t \in \Tb} \hloss_{k,t} - \ln K~.
\]
Putting together, rearranging, and using $1-\gamma \le 1$ gives
\[
    \sum_{t \in \Tb} \sum_{i \in V} p_{i,t}\hloss_{i,t}
\le
    \sum_{t \in \Tb} \hloss_{k,t} + \frac{2^b\ln K}{\gamma}
    + \gamma\sum_{t \in \Tb} \sum_{i \in R_t} \frac{\hloss_{i,t}}{|R_t|}
    + \frac{\gamma}{2^{b+1}} \sum_{t \in \Tb}\sum_{i \in V} p_{i,t}\bigl(\hloss_{i,t}\bigr)^2~.
\]
Reintroducing the notation $\gammab$ and summing over $b=0,1,\dots,\lfloor\log_2 K\rfloor$ gives
\begin{equation}\label{e:eq2}
    \sum_{t=1}^T \left( \sum_{i \in V} p^{(b_t)}_{i,t}\hloss^{(b_t)}_{i,t} - \hloss_{k,t} \right)
\le
    \sum_{b=0}^{\lfloor\log_2 K\rfloor}\frac{2^b\ln K}{\gammab}
    + \sum_{t=1}^T \sum_{i \in R_t} \frac{\gamma^{(b_t)}\hloss^{(b_t)}_{i,t}}{|R_t|}
    + \sum_{t=1}^T \frac{\gamma^{(b_t)}}{2^{b_t+1}} \sum_{i \in V}  p^{(b_t)}_{i,t}\bigl(\hloss^{(b_t)}_{i,t}\bigr)^2~.
\end{equation}
Now, similarly to the proof of Theorem~\ref{thm:noexp}, we have that, for any $i$ and $t$,
$
    \E_t\bigl[\hloss^{(b_t)}_{i,t}\bigr] = \loss_{i,t}
$
and
$
    \E_t\bigl[(\hloss^{(b_t)}_{i,t})^2\bigr] \leq \frac{1}{q^{(b_t)}_{i,t}}~.
$
Hence, taking expectations $\E_t$ on both sides of (\ref{e:eq2}) and recalling the definition of $Q^{(b)}_t$ gives
\begin{equation}\label{e:eq3}
    \sum_{t=1}^T \left( \sum_{i \in V} p^{(b_t)}_{i,t}\ell_{i,t} - \ell_{k,t} \right)
\le
    \sum_{b=0}^{\lfloor\log_2 K\rfloor}\frac{2^b\ln K}{\gammab}
    + \sum_{t=1}^T \sum_{i \in R_t} \frac{\gamma^{(b_t)}\ell_{i,t}}{|R_t|}
    + \sum_{t=1}^T \frac{\gamma^{(b_t)}}{2^{b_t+1}} Q^{(b_t)}_t~.
\end{equation}
Moreover,
\[
\sum_{t=1}^T \sum_{i \in R_t} \frac{\gamma^{(b_t)}\ell_{i,t}}{|R_t|}
\leq
\sum_{t=1}^T \sum_{i \in R_t} \frac{\gamma^{(b_t)}}{|R_t|}
=
\sum_{t=1}^T \gamma^{(b_t)}
=
\sum_{b=0}^{\lfloor\log_2 K\rfloor} \gammab|\Tb|
\]
and
\[
\sum_{t=1}^T \frac{\gamma^{(b_t)}}{2^{b_t+1}} Q^{(b_t)}_t
=
\sum_{b=0}^{\lfloor\log_2 K\rfloor} \frac{\gamma^{(b)}}{2^{b+1}} \sum_{t\in T^{(b)}} Q^{(b)}_t ~.
\]
Hence, plugging back into (\ref{e:eq3}), taking outer expectations on both sides and recalling
that $\Tb$ is random (since the adversary adaptively decides which steps $t$ fall into $\Tb$),
we get
\begin{align}
\nonumber
    \E\bigl[L_{A,T} - L_{k,T}\bigr]
& \le
    \sum_{b=0}^{\lfloor\log_2 K\rfloor}\E\left[\frac{2^b\ln K}{\gammab} + \gammab|\Tb|
    + \frac{\gammab}{2^{b+1}}\sum_{t \in \Tb} Q^{(b)}_t\right]
\\ &=
\label{eq:doubling}
    \sum_{b=0}^{\lfloor\log_2 K\rfloor} \left( \frac{2^b\ln K}{\gammab}
    +  \gammab\E\left[\sum_{t \in \Tb} \left(1 + \frac{Q^{(b)}_t}{2^{b+1}}\right)\right]\right)~.
\end{align}
%
%
This establishes~(\ref{eq:gammabfixed}).

In order to prove inequality~(\ref{eq:doublingtrick}), we need to tune each $\gammab$ separately.
However, a good choice of $\gammab$ depends on the unknown random quantity
\[
    \overline{Q}^{(b)} = \sum_{t \in \Tb} \left(1 + \frac{Q^{(b)}_t}{2^{b+1}}\right)~.
\]
To overcome this problem, we slightly modify Exp3-DOM by applying a doubling trick\footnote
{
The pseudo-code for the variant of Exp3-DOM using such a doubling trick
is not displayed in this extended abstract.
}
to guess $\overline{Q}^{(b)}$ for each $b$. Specifically, for each $b = 0, 1, \ldots, \lfloor \log_2 K \rfloor$,
we use a sequence
$\gammab_r = \sqrt{(2^b\ln K)/2^r}$, for $r=0,1,\dots$. We initially run the algorithm with $\gammab_0$.
Whenever the algorithm is running with $\gammab_r$ and observes that $\sum_s\overline{Q}^{(b)}_s > 2^r$,
where the sum is over all $s$ so far in $\Tb$,\footnote
{
Notice that $\sum_s\overline{Q}^{(b)}_s$ is an observable quantity.
}
then we restart the algorithm with $\gammab_{r+1}$. Because the contribution of instance
$b$ to ~(\ref{eq:doubling}) is
\[
    \frac{2^b\ln K}{\gammab} + \gammab\sum_{t \in \Tb} \left(1 + \frac{Q^{(b)}_t}{2^{b+1}}\right)~,
\]
the regret we pay when using any $\gammab_r$ is at most
$
    2\sqrt{(2^b\ln K) 2^r}
$.
The largest $r$ we need is $\bigl\lceil\log_2\overline{Q}^{(b)}\bigr\rceil$ and
\[
    \sum_{r=0}^{\lceil \log_2\overline{Q}^{(b)}\rceil} 2^{r/2} < 5\sqrt{\overline{Q}^{(b)}}~.
\]
Since we pay regret at most $1$ for each restart, we get
\[
    \E\bigl[L_{A,T} - L_{k,T}\bigr]
\le
    c\,\sum_{b=0}^{\lfloor\log_2 K\rfloor} \E\left[\sqrt{(\ln K)\left(2^b|\Tb| + \frac{1}{2}\sum_{t\in\Tb} Q^{(b)}_t\right)} + \bigl\lceil\log_2\overline{Q}^{(b)}\bigr\rceil\right]~.
\]
for some positive constant $c$.
Taking into account that
\begin{align*}
    \sum_{b=0}^{\lfloor\log_2 K\rfloor} 2^b|\Tb| &\le 2\sum_{t=1}^{T} |R_t|
\\
    \sum_{b=0}^{\lfloor\log_2 K\rfloor} \sum_{t \in \Tb} Q^{(b)}_t &= \sum_{t=1}^T Q_t^{(b_t)}
\\
    \sum_{b=0}^{\lfloor\log_2 K\rfloor} \bigl\lceil\log_2\overline{Q}^{(b)}\bigr\rceil
&= \mathcal{O}\bigl((\ln K)\ln(KT)\bigr)~,
\end{align*}
we obtain
\begin{align*}
    \E\bigl[L_{A,T} - L_{k,T}\bigr]
&\le
    c\,\sum_{b=0}^{\lfloor\log_2 K\rfloor}\E\left[\sqrt{(\ln K)\left(2^b|\Tb| + \frac{1}{2}\sum_{t\in\Tb} Q^{(b)}_t\right)}\right] + \mathcal{O}\bigl((\ln K)\ln(KT)\bigr)
\\ &\le
    c\,\lfloor\log_2 K\rfloor \E\left[\sqrt{\frac{\ln K}{\lfloor\log_2 K\rfloor}\sum_{t=1}^T\left(2|R_t| + \frac{1}{2} Q^{(b_t)}_t\right)}\right] + \mathcal{O}\bigl((\ln K)\ln(KT)\bigr)
\\
&=
{\mathcal O}\left((\ln K)\,\E\left[\sqrt{\sum_{t=1}^T \left(4|R_t|
+ Q_t^{(b_t)}\right)}\right] + (\ln K) \ln(KT)\right)
\end{align*}
%
as desired.
\end{proofof}
\begin{proofof}{Corollary~\ref{c:final}}
We start off from the upper bound~(\ref{eq:doublingtrick}) in the statement of Theorem~\ref{thm:alg}.
We want to bound the quantities $|R_t|$ and $Q_t^{(b_t)}$ occurring therein
at any step $t$ in which a restart does not occur ---the regret for the time steps when a restart occurs
is already accounted for by the term $\mathcal{O}\bigl((\ln K)\ln(KT)\bigr)$ in~(\ref{eq:doublingtrick}).
Now, Lemma~\ref{l:greedycover} gives 
\[
|R_t| = \mathcal{O}\bigl(\alpha(G_t)\ln K\bigr)~.
\]
If $\gamma_t = \gamma^{(b_t)}_t$ for any time $t$ when a restart does not occur,
it is not hard to see that $\gamma_t = \Omega\bigl(\sqrt{(\ln K)/(KT)}\bigr)$.
Moreover, Lemma~\ref{l:weightedamlemma} states that
\[
Q_t = \mathcal{O}\bigl(\alpha(G_t)\ln(K^2/\gamma_t) + |R_t|\bigr) = \mathcal{O}\bigl(\alpha(G_t)\ln(K/\gamma_t)\bigr)~.
\]
Hence, 
\[
Q_t = \mathcal{O}\bigl(\alpha(G_t)\ln(KT)\bigr).
\]
Putting together as in~(\ref{eq:doublingtrick}) gives the desired result.
\end{proofof}


\begin{thebibliography}{10}

\bibitem{as04}
N.~Alon and J.~H. Spencer.
\newblock {\em The probabilistic method}.
\newblock John Wiley \& Sons, 2004.

\bibitem{DBLP:conf/colt/AudibertB09}
Jean-Yves Audibert and S{\'e}bastien Bubeck.
\newblock Minimax policies for adversarial and stochastic bandits.
\newblock In {\em COLT}, 2009.

\bibitem{AuerCeFrSc02}
Peter Auer, Nicol\`o Cesa-Bianchi, Yoav Freund, and Robert~E. Schapire.
\newblock The nonstochastic multiarmed bandit problem.
\newblock {\em SIAM Journal on Computing}, 32(1):48--77, 2002.

\bibitem{c79}
Y.~Caro.
\newblock New results on the independence number.
\newblock In {\em Tech. Report, Tel-Aviv University}, 1979.

\bibitem{cb+97}
N.~Cesa-Bianchi, Y.~Freund, D.~Haussler, D.~P. Helmbold, R.~E. Schapire, and
  M.~K. Warmuth.
\newblock How to use expert advice.
\newblock {\em J. ACM}, 44(3):427--485, 1997.

\bibitem{cbl06}
N.~Cesa-Bianchi and G.~Lugosi.
\newblock {\em Prediction, learning, and games}.
\newblock Cambridge University Press, 2006.

\bibitem{Chv79}
V.~Chvatal.
\newblock A greedy heuristic for the set-covering problem.
\newblock {\em Mathematics of Operations Research}, 4(3):233--235, 1979.

\bibitem{FreundSc95}
Yoav Freund and Robert~E. Schapire.
\newblock A decision-theoretic generalization of on-line learning and an
  application to boosting.
\newblock In {\em Euro-COLT}, pages 23--37. Springer-Verlag, 1995.
\newblock Also, JCSS 55(1): 119-139 (1997).

\bibitem{Kalai:05}
A.~Kalai and S.~Vempala.
\newblock Efficient algorithms for online decision problems.
\newblock {\em Journal of Computer and System Sciences}, 71:291--307, 2005.

\bibitem{LittlestoneWa94}
Nick Littlestone and Manfred~K. Warmuth.
\newblock The weighted majority algorithm.
\newblock {\em Information and Computation}, 108:212--261, 1994.

\bibitem{MS11}
S.~Mannor and O.~Shamir.
\newblock From bandits to experts: On the value of side-observations.
\newblock In {\em 25th Annual Conference on Neural Information Processing
  Systems (NIPS 2011)}, 2011.

\bibitem{said2010social}
Alan Said, Ernesto~W De~Luca, and Sahin Albayrak.
\newblock How social relationships affect user similarities.
\newblock In {\em Proceedings of the International Conference on Intelligent
  User Interfaces Workshop on Social Recommender Systems, Hong Kong}, 2010.

\bibitem{vo90}
V.~G. Vovk.
\newblock Aggregating strategies.
\newblock In {\em COLT}, pages 371--386, 1990.

\bibitem{w81}
V.~K. Wey.
\newblock A lower bound on the stability number of a simple graph.
\newblock In {\em Bell Lab. Tech. Memo No. 81-11217-9}, 1981.

\end{thebibliography}
\end{document}